\documentclass[11pt]{article}
\usepackage{acronym}
\usepackage{mathrsfs}
\usepackage{hyperref}
\usepackage[dvipsnames]{xcolor}
\usepackage{rldm,palatino}
\usepackage{breqn,amsmath,amssymb,amsthm}
\usepackage[a4paper, total={6in, 9in}]{geometry}
\usepackage{latexsym}
\usepackage{cleveref}
\usepackage{nicefrac}
\usepackage[numbers]{natbib}
\usepackage[normalem]{ulem}

\hypersetup{
colorlinks=true,
citecolor=NavyBlue,
linkcolor=DarkOrchid,
}

\title{Tractable Representations for Convergent Approximation of Distributional HJB Equations}

\author{
Julie Alhosh\\
School of Computer Science\\
McGill University\\
Montr\'eal, Qu\'ebec, Canada\\
\texttt{julie.alhosh@mail.mcgill.ca} \\
\And
Harley Wiltzer \\
School of Computer Science\\
Mila, McGill University\\
Montr\'eal, Qu\'ebec, Canada\\
\texttt{harley.wiltzer@mail.mcgill.ca}\\
\And
David Meger\\
Centre for Intelligent Machines\\
School of Computer Science\\
McGill University\\
Montr\'eal, Qu\'ebec, Canada\\
\texttt{dmeger@cim.mcgill.ca} \\
}

\newtheorem{theorem}{Theorem}[section]
\newtheorem{assumption}{Assumption}
\newtheorem{definition}{Definition}[section]

\acrodef{MDP}{Markov decision process}
\acrodefplural{MDP}{Markov decision processes}
\acrodef{RL}{reinforcement learning}
\acrodef{CTRL}{continuous-time \ac{RL}}
\acrodef{HJB}{Hamilton-Jacobi-Bellman}
\acrodef{DHJB}{distributional \ac{HJB}}
\acrodef{SHJB}{statistical \ac{HJB}}

\newcommand{\that}{\hat{\ensuremath\tau}_i}

\newcommand{\thatp}{\hat{\ensuremath\tau}_{i+1}}
\newcommand{\finv}{\ensuremath \F^{-1}(x, \that)}
\newcommand{\finvone}{\ensuremath \F^{-1}(x, \hat{\tau}_1)}
\newcommand{\finvp}{\ensuremath \F^{-1}(x, \thatp)}
\newcommand{\heaviside}[1]{\vartheta_{#1}}
\newcommand{\N}{\ensuremath N}
\newcommand{\xz}{(x,z)}
\newcommand{\sxz}{(\vec{s}(x),z)}
\newcommand{\F}[1][\pi]{\ensuremath F_{\eta^{#1}}}
\newcommand{\K}{\mathrm{K}}
\newcommand{\imputation}[1][N]{\Phi_{#1}}
\newcommand{\imputationcdf}[1][N]{\imputation[#1]^{\mathsf{cdf}}}
\newcommand{\Kspace}[1][\imputation]{\K_{#1}^x}
\newcommand{\Kstat}[1][\imputation]{\K_{#1}^s}
\newcommand{\probset}[1]{\mathscr{P}(#1)}

\begin{document}

\maketitle

\begin{abstract}
In \ac{RL}, the long-term behavior of decision-making policies is evaluated based on their average returns.
Distributional \ac{RL} has emerged, presenting techniques for learning return \emph{distributions}, which provide additional statistics for evaluating policies, incorporating risk-sensitive considerations.
When the passage of time cannot naturally be divided into discrete time increments, researchers have studied the \emph{\ac{CTRL}} problem, where agent states and decisions evolve continuously. In this setting, the \ac{HJB} equation is well established as the characterization of the expected return, and many solution methods exist.
However, the study of distributional \ac{RL} in the continuous-time setting is in its infancy.
Recent work has established a \emph{\ac{DHJB}} equation, providing the first characterization of return \emph{distributions} in \ac{CTRL}.
These equations and their solutions are intractable to solve and represent exactly, requiring novel approximation techniques.
This work takes strides towards this end, establishing conditions on the method of parameterizing return distributions under which the \ac{DHJB} equation can be approximately solved.
Particularly, we show that under a certain topological property of the mapping between statistics learned by a distributional \ac{RL} algorithm and corresponding distributions, approximation of these statistics leads to close approximations of the solution of the \ac{DHJB} equation.
Concretely, we demonstrate that the \emph{quantile representation} common in distributional \ac{RL} satisfies this topological property, certifying an efficient approximation algorithm for continuous-time distributional \ac{RL}.
\end{abstract}


\acknowledgements{
We thank Marc Bellemare, Jesse Farebrother, Nate Rahn, and Hanna Yurchyk for their helpful feedback.
}  

\startmain 

\section{Introduction}\label{sec:intro}
Reinforcement Learning (RL) methods classically focus on evaluating policies by the returns they earn in expectation \citep{sutton2018reinforcement}.
The field of \emph{distributional RL} \citep{distributionalperspective,bellemare2023distributional} provides techniques for evaluating policies on the basis of their return \emph{distributions}.
Such techniques have demonstrated impressive empirical performance in various deep RL benchmarks \cite{distributionalperspective,wurman2022outracing}, and additionally present new possibilities for designing \emph{risk-sensitive} RL agents.

Research in \ac{RL} has traditionally focused on discrete-time problems, where time stops until a decision is made at fixed discrete timesteps.
In contrast, \acf{CTRL} is concerned with settings where decisions and state transitions evolve continuously in time.
In \ac{CTRL}, the expected return of a policy is governed by a differential equation known as the \acf{HJB} equation, which has been studied extensively in the \ac{RL} \citep{munos1997reinforcement} and optimal control \citep{fleming2006controlled} literature.
Recent works \citep{wiltzer2021evolution,harley,halperin2024distributional,wiltzer2024action} have studied distributional RL in the continuous-time setting. Particularly, a \emph{\acf{DHJB}} equation was introduced \citep{wiltzer2021evolution,harley}, which characterizes the CDF of return distributions in the \ac{CTRL} setting.
However, methods for solving the \ac{DHJB} equation are not as well understood.

In this work, we establish conditions under which solutions to the \ac{DHJB} equation can be approximated tractably and efficiently using familiar gradient-based iterative update schemes.
Since return distributions have infinitely-many degrees of freedom, they cannot be represented exactly on a computer, and therefore solutions to the \ac{DHJB} equation are intractable.
In response to this, Wiltzer et al. \citep{harley} introduced a \emph{\acf{SHJB} loss} as an objective for approximately solving the \ac{DHJB} equation by optimizing a finite set of return distribution statistics.
However, until this work, it has not been known whether minimization of the \ac{SHJB} loss truly yields close approximations to ground truth return distributions.
We make the following contributions:
\begin{enumerate}
    \item
    We prove that, as long as the \emph{imputation strategy}
    ---that is, the mapping from statistics to return distributions---satisfies a certain topological property, minimizing the \ac{SHJB} loss yields convergent approximations of the return distributions;
    \item We demonstrate that the \emph{quantile} distribution representation \citep{quantileregression} and its corresponding imputation strategy satisfy this topological property, providing a principled loss function for continuous-time distributional RL.
\end{enumerate}

\section{Background}
In this section, we briefly summarize the relevant prior results of \citep{harley}, in the continuous-time infinite-horizon discounted return setting. The state space is denoted $\mathcal{X}$ and is assumed to be bounded and Euclidean, and the discount factor is $\gamma\in(0,1)$.
We assume the reward function $r$ is bounded, so that returns are confined to a bounded set $\mathcal{R} = [V_\min, V_\max]$.

We will consider a fixed policy $\pi$, and our goal is to estimate its \emph{return distribution function} 
$
{
\eta^\pi:\mathcal{X}\to\probset{\mathcal{R}},
}
$
which is the distribution of
\begin{equation}\label{eq:rdf}
\int_0^\infty\gamma^tr(X_t)\mathrm{d}t =: G^\pi(x)\sim\eta^\pi(x), \quad X_0 = x, \quad \mathrm{d}X_t = \mu_\pi(X_t)\mathrm{d}t + \sigma_\pi(X_t)\mathrm{d}B_t. 
\end{equation}
Here, $(B_t)_{t\geq 0}$ is a Brownian motion, and $\mu_\pi, \sigma_\pi$ describe the stochastic differential equation governing the agent's state under the policy $\pi$.
Recall that the \emph{value function} $V^\pi$ given by
\begin{equation}
V^\pi:x\mapsto\mathbf{E}[G^\pi(x)]
\end{equation} 
is characterized by the \acf{HJB} equation \cite{fleming2006controlled},
\begin{equation}
    \label{eq:hjb}
    \langle\nabla_xV^\pi(x), \mu_\pi(x)\rangle + \log\gamma V^\pi(x) + r(x) + \frac{1}{2}\mathrm{Tr}\left(\sigma_\pi(x)^\top(\nabla_x^2V^\pi(x))\sigma_\pi(x)\right)= 0.
\end{equation}

Henceforth, we will write $\F(x, \cdot)$ to denote the CDF of $\eta^\pi(x)$.
To begin, we formalize how imputation strategies encode distribution approximations.

\begin{definition}[Imputation Strategy]
An \emph{imputation strategy} is a map 
$
\imputation:\mathbf{R}^\N\to\probset{\mathcal{R}}
$
which maps a set of $\N$ statistics into a probability measure with those statistics.
Additionally, for every imputation strategy $\imputation$, we define a map 
$
\imputationcdf:\mathbf{R}^\N\times\mathcal{R}\to[0,1]
$
given by 
\begin{equation}
\imputationcdf(s, z) = [\imputation(s)]([V_\min, z]).
\end{equation}
\end{definition}

As an example, we may consider a \emph{Gaussian} imputation strategy, such as
\begin{equation}
\imputation[2]:(\mu, \sigma^2)\mapsto\mathcal{N}(\mu, \sigma^2),
\end{equation} 
which is a valid imputation strategy when $\mu, \sigma^2$ are interpreted as mean and variance statistics.
Alternatively, for a fixed subset $\{\xi_i\}_{i=1}^N\subset\mathcal{R}$, the map 
\begin{equation}
\imputation: p\mapsto \sum_{i=1}^Np_i\delta_{\xi_i}
\end{equation} 
is an imputation strategy for statistics $p$ in the simplex $\Delta_N$, which maps a set of probabilities to a categorical distribution on the support $\{\xi_i\}_{i=1}^N$.
Generally, $\N$ controls the \emph{resolution} of our distribution approximations; as $\N$ increases, we hope to achieve higher fidelity approximations, and to more closely solve the \ac{DHJB} equation.

Next, we recall some useful definitions regarding generalized functions.

\begin{definition}[Schwartz Class] \label{def:schwartz}
Let $X$ be a normed space.
A Schwartz class is a class $\mathcal{S}$ of rapidly decaying-smooth functions 
\begin{equation}
\mathcal{S} = \{ f \in C^\infty (X;\mathbf{R}) : \  \sup_{x \in X} \big[(1 + \Vert x \Vert^k) |f^{(m)} (x)|\big] < \infty \quad \forall k, m \in \mathbf{N}\}.
\end{equation}
\end{definition}

\begin{definition}[Tempered Distribution]\label{def:tdist} 
A tempered distribution is an element of the topological dual $\mathcal{S}'$ of the Schwartz class $\mathcal{S}$.
That is, a tempered distribution is a linear map 
$
\varrho:\mathcal{S}\to\mathbf{R}.
$
\end{definition}

It is important to note that ``Distribution'', in the context of a tempered distribution, refers to a type of generalized function, and not a \emph{probability} distribution.

This formalism allows us to generalize notions of solutions to differential equations, permitting solutions where derivatives are tempered distributions. For a differential operator $\mathscr{L}$, we say a differential equation $\mathscr{L}f = 0$ \emph{in the distributional sense} if 
\begin{equation}
\int\phi(y)(\mathscr{L}f)(y)\mathrm{d}y = 0 \quad \forall \phi\in\mathcal{S}.
\end{equation}

We are now able to state the main regularity conditions assumed on the imputation strategy.
\begin{definition}[Statistical Smoothness \cite{harley}]\label{def:statsmooth} 
An imputation strategy $\imputation$ is said to be \emph{statistically smooth} if $\imputation(s)$ is a tempered distribution for each $s \in \mathbf{R}^\N$.
Likewise, a return distribution function 
$
{
\eta:\mathcal{X}\to\probset{\mathcal{R}}
}
$
is said to be statistically smooth if its state-conditioned CDF, $\F(x,\cdot)$ is a tempered distribution for each $x \in \mathcal{X} $ and $\F(\cdot,z)$ is twice continuously differentiable almost everywhere for each $z\in \mathcal{R}$.
\end{definition}

\begin{assumption}\label{assumption:1}
    At every state $x \in \mathcal{X}$, $\eta^\pi(x)$ is absolutely continuous with respect to the Lebesgue measure.
\end{assumption} 
    
\begin{assumption}\label{assumption:2}
    The mapping 
    $
    \xz \mapsto \F\xz
    $
    is twice differentiable over $\mathcal{X} \times \mathcal{R}$ almost everywhere, and its second partial derivatives are continuous almost everywhere.
\end{assumption} 

Having established these technical assumptions, we can state the \ac{DHJB} equation, which characterizes return distributions in \ac{CTRL}.
Finally, we will recall the \ac{SHJB} loss, which is our central focus.

\begin{theorem}[\ac{DHJB} Equation, \cite{harley}]
    \label{thm:dhjb}
    Under assumptions \ref{assumption:1} and \ref{assumption:2}, for almost every $(x, z)\in\mathcal{X}\times\mathcal{R}$, the \acf{DHJB} equation 
    \begin{equation}
    (\mathscr{L}\F)(x, z) = 0
    \end{equation}
    holds in the distributional sense, where
    \begin{equation}
        \label{eq:dhjb}
        (\mathscr{L}f)(x, z)
        =\langle\nabla_xf(x, z), \mu_\pi(x)\rangle - (r(x) + z\log\gamma)\frac{\partial}{\partial z}f(x, z) + \frac{1}{2}\mathrm{Tr}\left(\sigma_\pi(x)^\top (\nabla_x^2f(x, z))\sigma_\pi(x)\right).
    \end{equation}
\end{theorem}

The characterization of return distributions from \Cref{thm:dhjb} gives rise to a \emph{\acf{SHJB}} loss for evaluating finitely-parameterized return distribution approximations.
\begin{theorem}[\ac{SHJB} Loss, \cite{harley}]
\label{thm:shjb}
Let assumptions \ref{assumption:1} and \ref{assumption:2} hold, and let $\imputation$ be a statistically smooth imputation strategy.
If 
\begin{equation}
\F : (x, z) \mapsto \imputationcdf(\vec{s}(x), z)
\end{equation} 
for a differentiable \emph{statistics function} $\vec{s}:\mathcal{X}\to\mathbf{R}^\N$, then the \ac{SHJB} loss $\mathcal{L}_S(\vec{s}, \imputationcdf)$ vanishes, where
\begin{equation}\begin{split}
    \mathcal{L}_S(\vec{s},\imputationcdf) = \big[ \nabla_{\vec{s}(x)} \imputationcdf(\vec{s}(x),z)^\top \vec{s}(x) \mu_\pi(x)      & -(r(x) + z\log\gamma) \frac{\partial}{\partial z}\imputationcdf(\vec{s}(x),z)\\ &+\frac{1}{2} \mathrm{Tr}(\sigma_\pi(x)^\top (\Kspace\xz+\Kstat\xz) \sigma_\pi(x)) \big]^2,
\end{split}\end{equation}

and denoting by $J_x$ the Jacobian with respect to $x$,
\begin{equation*}\begin{split}
    \Kspace \xz = \sum^\N_{k=1} \frac{\partial}{\partial s_k(x)} \imputationcdf\sxz \nabla^2_{x} \vec{s}_k(x)
    \hspace{.5em}\mathrm{and}\hspace{.5em}
    \Kstat \xz =  J_x \vec{s}(x) ^\top(\nabla^2_{\vec{s}} \imputationcdf\sxz) J_x \vec{s}(x).
\end{split}\end{equation*}
\end{theorem}

\section{Convergence  of the Statistical HJB Loss} \label{sec:convergence}
In this section, specifically in \Cref{thm:HJBlossconv}, we prove that the \ac{SHJB} loss $\mathcal{L}_S$ converges to $0$ as the number of statistics $\N$ approaches $\infty$, when the imputation strategy $\imputation$ satisfies a certain topological property.
As mentioned in \cref{sec:intro}, both $\imputationcdf$ and the return distribution CDF $\F$ are assumed to be tempered distributions.

\begin{theorem}[Convergence of \Ac{SHJB} Loss]\label{thm:HJBlossconv}
Let $\{\Phi_\N\}_{\N=0}^\infty$ be a sequence of statistically smooth imputation strategies, and let $\vec{s}: \ \mathcal{X} \to \mathbb{R}^\N$ be a twice continuously differentiable statistics function.
Let assumptions in \Cref{thm:shjb} hold.
If  $\imputationcdf(\vec{s}(x), \cdot)$ converges to $\F(x, \cdot)$ in the space of tempered distributions for each $x \in \mathcal{X}$, then it holds in the distributional sense that
\begin{equation}
\lim_{\N\to\infty}\mathcal{L}_S\big(\vec{s},\imputationcdf\big) =0.
\end{equation} 
\end{theorem}

\begin{proof}
Express $\mathcal{L}_S(\vec{s},\imputationcdf) = (g(\vec{s}, \imputationcdf))^2$ for 
\begin{equation}\begin{split}
    g(\vec{s}, \imputationcdf) = \Bigg[\nabla_{\vec{s}(x)} \imputationcdf\sxz^\top \vec{s}_x(x) \mu_\pi(x) & -(r(x) + z\log\gamma) \frac{\partial}{\partial z}\imputationcdf\sxz\\ 
    &+\frac{1}{2}\mathrm{Tr}(\sigma_\pi(x)^\top (\Kspace\xz+\Kstat\xz) \sigma_\pi(x))\Bigg].
\end{split}\end{equation}

Note that the distributional derivative is a continuous operator. By taking the assumptions of \Cref{thm:HJBlossconv}
into consideration, we have that $g(\vec{s}, \imputationcdf)$ is a linear combination of continuous functions and is therefore continuous.
Then, the continuity of $\mathcal{L}_S$ follows from the continuity rule for composite functions.
Now, by continuity, we have
\begin{equation}
    \lim_{\N \to \infty} \mathcal{L}_S(\vec{s},\imputationcdf) = \mathcal{L}_S(\vec{s},\lim_{\N \to \infty}\imputationcdf).
\end{equation}

Moreover, by our assumption that $\imputationcdf(\vec{s}(x), \cdot)$ converges to $\F(x, \cdot)$, and since $\F$ satisfies the \ac{DHJB} equation, we have 
\begin{equation}
\mathcal{L}_S(\vec{s}, \lim_{N\to\infty}\imputationcdf) =\mathcal{L}_S(\vec{s}, \F) =  0
\end{equation} 
in the distributional sense, by \Cref{thm:shjb}.
\end{proof}

\section{Quantile Approximation of Return Distributions}\label{sec:quantcase}
Choosing quantiles as the statistics used to compute the imputation strategy is common in distributional \ac{RL} \citep{quantileregression,harley}.
In this section, \Cref{thm:quantconv} proves that the corresponding quantile imputation strategy satisfies the hypotheses of \Cref{thm:HJBlossconv}.
As a result, we prove in \Cref{thm:quantlossconv} that the \ac{SHJB} loss vanishes as we increase the number $\N$ of quantiles in our distribution representation.

We denote the inverse CDF of $\eta^\pi(x)$ via 
$\F^{-1}(x, \cdot)$, given by 
\begin{equation}
\F^{-1}(x, \tau) = \inf_{z\in\mathcal{R}}\{\F(x, z)\geq \tau\},
\end{equation} 
and the $\tau$-quantile of $\eta^\pi(x)$ is $\F^{-1}(\tau)$ for $\tau\in(0,1)$. The \emph{quantile imputation strategy} \citep{quantileregression,bellemare2023distributional} is given by
\begin{equation}
\label{quantiledist}
\imputationcdf\sxz = \frac{1}{\N} \sum_{i=1}^{\N} \heaviside{\vec{s}_i(x)}(z),
\end{equation}
where $\heaviside{y}$ is the Heaviside step function at $y\in\mathbf{R}$, and $\vec{s}_i(x) = \F^{-1}(x, \nicefrac{2i-1}{2\N})$.

\begin{theorem}[Convergence of the Quantile Imputation Strategy]\label{thm:quantconv}
The sequence of quantile imputation strategies $\{\imputationcdf(\vec{s}(x), \cdot)\}_{\N=1}^\infty$ converges to $\F(x, \cdot)$ in the space of tempered distributions for each $x \in \mathcal{X}$.
\end{theorem}

\begin{proof}
Henceforth, we denote \begin{equation}
\tau_i = \frac{i}{\N} \ \ \text{ and } \ \  \that := \frac{\tau_{i-1} + \tau_i}{2}=\frac{2i-1}{2\N} \qquad i \in [\N]
\end{equation}
Let $\N \in \mathbf{N}$ be arbitrary, we will show that $\forall \xz\in(\mathcal{X},\mathcal{R})$, 
\begin{equation}
{\left| \F\xz - \imputationcdf\sxz\right| \leq \frac{1}{2\N}}.
\end{equation}
Then, we will prove the sequence $\{\imputationcdf\}_{\N=1}^\infty$ converges to $\F$ by showing that $\forall \epsilon > 0, \ \exists n\in\mathbf{N}$ such that $\forall \N\geq n$, 
\begin{equation}
\|\imputationcdf(\vec{s}(x), \cdot) - \F(x, \cdot)\|<\epsilon.
\end{equation}

Fix any $x \in \mathcal{X}$ and $i \in [\N]$.
If $\finv\neq\finvp$, then
$\forall z \in \left[\finv, \finvp\right)$, $\imputationcdf\sxz= \tau_{i} \ $ and we have that:
\begin{eqnarray}
\that \leq                      & \F\xz                     & \leq \thatp\label{eq:ineq} \\
 \tau_{i}-\frac{1}{2\N}  \leq   & \F\xz                     & \leq   \tau_{i} + \frac{1}{2\N} \\
-\frac{1}{2\N}  \leq            & \F\xz-\tau_{i}               & \leq \frac{1}{2\N}  \\ 
                                & \left| \F\xz-\imputationcdf\sxz \right|  &\leq \frac{1}{2\N}  
\end{eqnarray}
where inequality \eqref{eq:ineq} holds since $\F$ is an increasing function by definition.
Otherwise, there exists $i \in [\N]$ such that $\finv = \finvp$.
Then
\begin{equation}\label{eq:contradiction}\begin{split}
    \eta^\pi(x)(\{\finvp\} ) &= \eta^\pi(x)(\{\finv\})\\
    &= \eta^\pi(x)(\left[V_\min,\finvp\right]) - \eta^\pi(x)(\left[V_\min,\finv\right))\\
    &= \F(x, \finvp) - \eta^\pi(x)(\left[V_\min,\finv\right)) \\
    & \geq \thatp - \that \\
    & > 0.
\end{split}\end{equation}

But since $\F$ is absolutely continuous by \Cref{assumption:1}, we must have $\eta^\pi(x)(\{\finvp\} ) = 0$, contradicting \eqref{eq:contradiction}.
Therefore, $\forall i \in [\N], \ \finv$ is unique.
Similarly, we obtain the same result for $z \in \left[V_\min, \finvone\right)$.
Now, we have that $\forall x\in\mathcal{X}$,
\begin{equation}\begin{split}
    \|\imputationcdf(\vec{s}(x), \cdot) - \F(x, \cdot)\| & =  \sup_{\|\phi\|=1} \int_{\mathcal{R}}\left( \F\xz- \imputationcdf\sxz \right)\phi(z) \ \mathrm{d}z  \\
    & \leq \frac{1}{2\N}\sup_{\|\phi\|=1} \int_{\mathcal{R}}\phi(z) \ \mathrm{d}z.
\end{split}\end{equation}
Since $\phi \in \mathcal{S}$, it is bounded, and
since $\mathcal{X}, \mathcal{R}$ are compact, there exists a finite $M\in\mathbf{R}_+$ such that $\sup_{\|\phi\|=1}\int_{\mathcal{R}}\phi(z)\mathrm{d}z\leq M$.
So, $\forall \epsilon > 0$, it holds that for $\N> \nicefrac{M}{2\epsilon}$, 
\begin{equation}
\|\imputationcdf(\vec{s}(x), \cdot)- \F(x, \cdot)\| < \epsilon.
\end{equation}
\end{proof}
 
\begin{theorem}[Convergence of \ac{HJB} Loss in the Quantile Case]\label{thm:quantlossconv}
Let $\vec{s}_i(x)$ denote the $\nicefrac{2i-1}{2\N}$-quantile of $\eta^\pi(x)$, and suppose the quantile map $\vec{s}$ is twice continuously differentiable. Then under assumptions \ref{assumption:1} and \ref{assumption:2}, if $\imputation$ is the quantile imputation strategy, it holds in the distributional sense that
\begin{equation}
\lim_{\N\to\infty}\mathcal{L}_S(\vec{s}, \imputation) = 0.
\end{equation} 
\end{theorem}

\begin{proof}
By \Cref{thm:quantconv}, we know that $\imputationcdf(\vec{s}(x), \cdot)$ converges to $\F(x, \cdot)$ as $\N$ approaches $\infty$.
Thus, by our hypotheses, we may apply \Cref{thm:HJBlossconv}, asserting that the \ac{SHJB} loss converges to $0$ as $\N$ approaches $\infty$.
\end{proof}

This result has important implications for approximately solving distributional HJB equations.
Notably, it certifies that minimizing the \ac{SHJB} loss yields close approximations to the solution of the \ac{DHJB} equation, under a novel and simple condition on the imputation strategy.
Particularly, we validated that minimization of the \ac{SHJB} loss yields close approximations to $\eta^\pi$ under the quantile imputation strategy; this nicely complements the results of \citep{harley}, which provides an algorithm for minimizing the \ac{SHJB} loss over quantile representations.
Prior to this work, \citet{harley} had shown only that the \ac{SHJB} loss is $0$ when $\eta^\pi(x)$ could be exactly represented by $\imputation(\vec{s}(x))$ for some finite $\N$---our results show that this loss is useful even in the much more realistic setting where $\imputation(\vec{s}(x))$ is only an approximation of $\eta^\pi(x)$.

\bibliographystyle{plainnat}
\bibliography{references.bib}

\end{document}